\documentclass{article} 
\usepackage{preprint,times}

\usepackage{amsmath,amsfonts,bm}
\usepackage{booktabs}
\usepackage{bbm}
\usepackage{graphicx}
\usepackage{subcaption}
\usepackage{algorithm}
\usepackage{algpseudocode}
\usepackage{acronym}
\usepackage{mathtools}
\usepackage{amsthm}
\usepackage{wrapfig}
\usepackage{caption}
\usepackage{float}
\usepackage{tikz}
\usetikzlibrary{arrows.meta,decorations.pathreplacing}

\newtheorem{proposition}{Proposition}

\newcommand{\gradstrut}{\vphantom{\sum_{i=0}^{k-1}\frac{z_j-\lambda}{\lambda(1-\lambda)}}}
\acrodef{RL}{reinforcement learning}
\acrodef{MPC}{model predictive control}
\acrodef{BC}{behavior cloning}
\acrodef{PPO}{proximal policy optimization}
\acrodef{MDP}{Markov decision process}
\acrodef{POMDP}{partially observable Markov decision process}
\acrodef{IL}{imitation learning}
\acrodef{IRL}{inverse reinforcement learning}
\acrodef{DAgger}{dataset aggregation}
\acrodef{SAC}{soft actor-critic}
\acrodef{TD3}{twin delayed deep deterministic policy gradient}
\acrodef{DDPG}{deep deterministic policy gradient}
\acrodef{TRPO}{trust region policy optimization}
\acrodef{GAE}{generalized advantage estimation}
\acrodef{DQN}{deep Q-network}
\acrodef{A2C}{advantage actor-critic}
\acrodef{DPO}{direct preference optimization}
\acrodef{RLHF}{reinforcement learning from human feedback}
\acrodef{MBRL}{model-based reinforcement learning}
\acrodef{MFRL}{model-free reinforcement learning}
\acrodef{DRL}{deep reinforcement learning}
\acrodef{MPPI}{model predictive path integral}
\acrodef{sim2real}{simulation-to-reality}

\acrodef{OCP}{optimal control problem}
\acrodef{NMPC}{nonlinear model predictive control}
\acrodef{LMPC}{linear model predictive control}
\acrodef{RMPC}{robust model predictive control}
\acrodef{SMPC}{stochastic model predictive control}
\acrodef{LQR}{linear quadratic regulator}
\acrodef{iLQR}{iterative linear quadratic regulator}
\acrodef{DDP}{differential dynamic programming}
\acrodef{DP}{dynamic programming}
\acrodef{ADP}{approximate dynamic programming}
\acrodef{NLP}{nonlinear program}
\acrodef{QP}{quadratic program}
\acrodef{SQP}{sequential quadratic programming}
\acrodef{IP}{interior point}
\acrodef{RTI}{real-time iteration}
\acrodef{KKT}{Karush-Kuhn-Tucker}
\acrodef{ODE}{ordinary differential equation}
\acrodef{RK4}{fourth-order Runge-Kutta}
\acrodef{PID}{proportional-integral-derivative}
\acrodef{CBF}{control barrier function}
\acrodef{CLF}{control Lyapunov function}
\acrodef{HJB}{Hamilton-Jacobi-Bellman}

\acrodef{NN}{neural network}
\acrodef{MLP}{multilayer perceptron}
\acrodef{DNN}{deep neural network}
\acrodef{RNN}{recurrent neural network}
\acrodef{LSTM}{long short-term memory}
\acrodef{GP}{Gaussian process}
\acrodef{VAE}{variational autoencoder}
\acrodef{SGD}{stochastic gradient descent}
\acrodef{MSE}{mean squared error}
\acrodef{KL}{Kullback-Leibler}
\acrodef{ML}{machine learning}
\acrodef{AI}{artificial intelligence}

\acrodef{JSRL}{Jump-Start Reinforcement Learning}

\newcommand{\cS}{\mathcal{S}}
\newcommand{\cA}{\mathcal{A}}

\DeclareMathOperator*{\Ex}{\mathbb{E}}

\def\eqref#1{equation~\ref{#1}}

\def\1{\bm{1}}

\DeclareMathAlphabet{\mathsfit}{\encodingdefault}{\sfdefault}{m}{sl}
\SetMathAlphabet{\mathsfit}{bold}{\encodingdefault}{\sfdefault}{bx}{n}

\def\gR{{\mathcal{R}}}

\newcommand{\E}{\mathbb{E}}

\newcommand{\R}{\mathbb{R}}

\newcommand{\parag}[1]{\vspace{-3mm}\paragraph{#1}}

\newcommand{\piE}{\hat{\pi}}
\newcommand{\pinn}{\pi_\theta}
\newif\ifdraft
\draftfalse

\ifdraft
\newcommand{\PF}[1]{{\color{red}{\bf PF: #1}}}

\newcommand{\MX}[1]{{\color{blue}{\bf MX: #1}}}

\newcommand{\DA}[1]{{\color{teal}\textbf{DA: #1}}}

\else
\newcommand{\PF}[1]{}

\newcommand{\MX}[1]{}

\newcommand{\DA}[1]{}

\fi

\usepackage{hyperref}
\usepackage{url}

\title{Adaptive Expert Guidance for Efficient On-Policy Reinforcement Learning}

\author{%
Daniele Affinita$^{1}$ \quad
Ming Xu$^{1}$ \quad
Rudolf Reiter$^{2}$ \quad
Davide Scaramuzza$^{2}$ \quad
Pascal Fua$^{1}$ \\[0.5ex]
{\normalfont $^{1}$EPFL, Switzerland \qquad $^{2}$University of Zurich, Switzerland}
}

\begin{document}

\maketitle


\begin{abstract}
With massively parallel simulation, on-policy \ac{RL} methods such as \ac{PPO} have become standard in many domains. However, learning from scratch is sample-inefficient and fails to exploit the potential existence of a suboptimal expert, such as a heuristic, a model-based controller, or a policy trained on a related task.
Such an expert is often available and can guide early training, but its sub-optimality limits final performance. The challenge then becomes balancing expert guidance against learning from rewards.
Existing methods set the expert's influence through a blending weight, a schedule, or an evaluation-driven curriculum. Alternatively, they adapt it with additional learned components such as critics over expert actions or auxiliary agents. However, none optimizes it using the same on-policy objective as the policy itself.
We propose a method in which the learner and the expert alternate control within each training episode, and the expert's share of control is a single learnable parameter optimized jointly with the policy. The learner benefits from the expert early in training, but its share of control declines as the learner becomes more competent, until eventually vanishing completely. This leaves the learner acting alone and better than the suboptimal expert.  
We evaluate our method on 34 tasks across two benchmarks, spanning discrete and continuous action spaces, using both learned and model-based experts. Our method improves sample efficiency over guided and unguided baselines while requiring minimal hyperparameter variation. The expert's share decays to zero as the learner improves, vanishing when the expert is no longer useful.
\end{abstract}

\acresetall


\section{Introduction}

\Ac{RL} has demonstrated remarkable ability to solve complex sequential decision-making problems~\citep{silver2018alphazero, andrychowicz2020learning, kaufmann2023champion, guo2025deepseek}, but it requires many environment interactions, as the agent must explore the usually high-dimensional state space before it can discover and exploit high-reward behaviors. GPU-parallel simulation~\citep{makoviychuk2021isaac} favors on-policy methods such as \ac{PPO}~\citep{schulman2017proximal}, because a single GPU can step thousands of environments in parallel and cut the wall-clock cost of collecting experience. Off-policy and model-based methods need fewer interactions~\citep{haarnoja2018soft, hansen2024tdmpc2} but require more computation per interaction. Their higher per-interaction cost therefore limits the wall-clock gains they can obtain from massively parallel simulation. In this paper, we therefore focus on on-policy methods.

In practice, training does not need to start from scratch~\citep{agarwal2022reincarnating} when an \emph{expert} is available, that is, an algorithm that delivers a useful, though not necessarily optimal, policy. Examples include controllers and planners built on approximate physical models~\citep{carius2020mpcnet}, hand-engineered policies~\citep{uchendu2023jumpstart}, previously trained agents~\citep{rusu2016policy, schmitt2018kickstarting}, and agents trained on a related task~\citep{taylor2009transfer, zhu2023transfer}. Imitating the expert's actions limits the learner to the expert's quality~\citep{ross2011dagger, carius2020mpcnet}. Other methods keep the expert in the control loop, and therefore in the deployed system, through residual control~\citep{silver2018residual, johannink2019residual, jeon2025residualmpc}, uncertainty-weighted fusion of expert and learner actions~\citep{rana2020multiplicative, rana2023bayesian}, \ac{MPC} as an optimization layer~\citep{reiter2023hilepp, romero2025acmpc, adabag2025diffmpcgpu}, or \ac{MPC} whose value estimates are blended with a learned value function~\citep{bhardwaj2021blending}.

Since the expert is suboptimal, training should exploit its guidance early and then hand control to the learner, so that the learner can improve beyond the expert and act alone after training. Most existing approaches set the expert's influence independently of the learning objective. Expert-regularized \ac{RL} decays the imitation weight on a schedule~\citep{rajeswaran2018dexterous} or tunes it by population-based training~\citep{schmitt2018kickstarting}. When the expert contributes experience during training, its share is shortened on an evaluation-driven curriculum~\citep{uchendu2023jumpstart}, decayed on a schedule~\citep{yang2023reinforcement}, or kept at a fixed ratio of pre-collected expert traces to learner samples~\citep{yan2025learning}. Methods that adapt the expert's role to its value rely on a learned selector or a critic over the expert's actions~\citep{nair2018overcoming, xie2018learning, kurenkov2019ac, zhang2023policyexpansion}, or on a counterfactual agent that learns from rewards alone~\citep{shenfeld2023tgrl}.

We instead propose a guided policy-gradient algorithm in which the expert and learner share control during training, with the expert's share represented by a single parameter optimized under the same return objective as the policy. We show that its gradient corresponds to the \emph{expert advantage}: the expected gain in return from letting the expert act instead of the learner, from the same states. Naively learning this share is ineffective: switching policies at every step makes the comparison too noisy, while optimizing return alone does not necessarily withdraw an expert. We therefore commit to one policy for blocks of $k$ steps, yielding a more reliable comparison, and introduce a price on expert use, making withdrawal part of the objective. An imitation term on expert blocks further trains the learner whenever the expert acts. Once the learner catches up, the expert's share decreases to zero and training reduces to standard policy gradient.

%
The proposed algorithm is evaluated on 34 tasks across two common benchmarks spanning discrete and continuous action spaces, with minimal hyperparameter variation. The experts, learned or model-based, range in quality from near-optimal to substantially worse than the converged learner. Our method improves sample efficiency over guided and unguided baselines, and the learner asymptotically recovers the performance of  Proximal Policy Optimization (\ac{PPO}) or better, so the expert's suboptimality does not limit the final policy. The rate of expert withdrawal adapts to expert quality, with weaker experts withdrawn faster.\PF{ And how do you compare to state-of-the-art methods?} 


\section{Related Work}

We group prior work by how the expert enters training: as a source of target actions (imitation), as an extra term in the objective (regularization), or as a policy that collects part of the rollouts. Each group must decide how strongly and for how long the expert should influence the learner, the question that our method answers with a learned expert share. Our method belongs to the third group, and the last paragraph relates its block structure to temporal abstraction.

\parag{Imitating an expert.}
Behavior cloning suffers from compounding errors off the demonstration distribution, which \ac{DAgger}~\citep{ross2011dagger} corrects by querying the expert on the states the learner visits, and a learner that regresses onto a fixed expert's actions inherits the expert's ceiling. Guided policy search~\citep{levine2013guided, levine2016endtoend} and expert iteration~\citep{anthony2017thinking, silver2018alphazero} avoid that ceiling with a teacher that improves during training, whereas our expert is fixed and imitation acts only as a regularizer that vanishes once the expert is withdrawn. Other methods surpass a fixed expert by combining imitation with reward-driven updates: LOKI~\citep{cheng2018fast} imitates for a randomized number of iterations before switching to policy gradient, while MAMBA~\citep{cheng2020policy} and RPI~\citep{liu2024rpi} imitate black-box oracles state-wise, wherever an oracle's estimated value exceeds the learner's. RPI is the closest of these to our setting, but its handover compares value functions learned from each oracle's rollouts, whereas we use the learner's own critic on the training rollouts.

\parag{Expert-regularised reinforcement learning.}
A second family adds an expert-derived term to the reward objective: a large-margin~\citep{hester2018dqfd} or behavior-cloning~\citep{nair2018overcoming} loss on demonstrations, an imitation term in the policy gradient~\citep{rajeswaran2018dexterous}, or distillation from a teacher on the student's rollouts~\citep{schmitt2018kickstarting, agarwal2022reincarnating}. The learner can surpass the expert, so the design question is the regularizer weight. The Q-filter of \citet{nair2018overcoming} applies imitation only where the critic prefers the demonstrated action, and TGRL~\citep{shenfeld2023tgrl} adapts the weight by comparing the guided agent with a reward-only counterfactual agent.

\parag{Experts that act during training.}
Several methods let the expert control part of the rollout and shrink that part over training: a prefix of each episode shortened by a curriculum in \ac{JSRL}~\citep{uchendu2023jumpstart}, or a probability of acting decayed on a schedule~\citep{yang2023reinforcement}. Others arbitrate between expert and learner with a learned selector or critic, at every step~\citep{xie2018learning, zhang2023policyexpansion} or with commitment until a Bayesian critic favors switching~\citep{kurenkov2019ac}. In our method, blocks have fixed length and the switch ignores the state, so the share gradient reduces to a comparison of expert and learner block advantages, which we estimate from the training rollouts with the learner's own critic. No selector, critic over expert actions, evaluation phase, or auxiliary agent is needed, and the share follows the same objective as the learner, so withdrawal is part of it.

\parag{Temporal abstraction.}
Our method can be connected to the options framework~\citep{sutton1999between}, in which an agent chooses among temporally extended actions. Each of our blocks can be seen as the execution of one of two such options, the expert or the learner, which last a fixed number of steps and of which only the learner is trained. Under this view, our learner update corresponds to the intra-option policy gradient of option-critic~\citep{bacon2017option}, which learns a set of options that together solve the task, whereas we learn to withdraw one of ours.


\section{Method}
\label{sec:method}

\begin{figure}[t]
    \centering
    \includegraphics[width=1.0\linewidth]{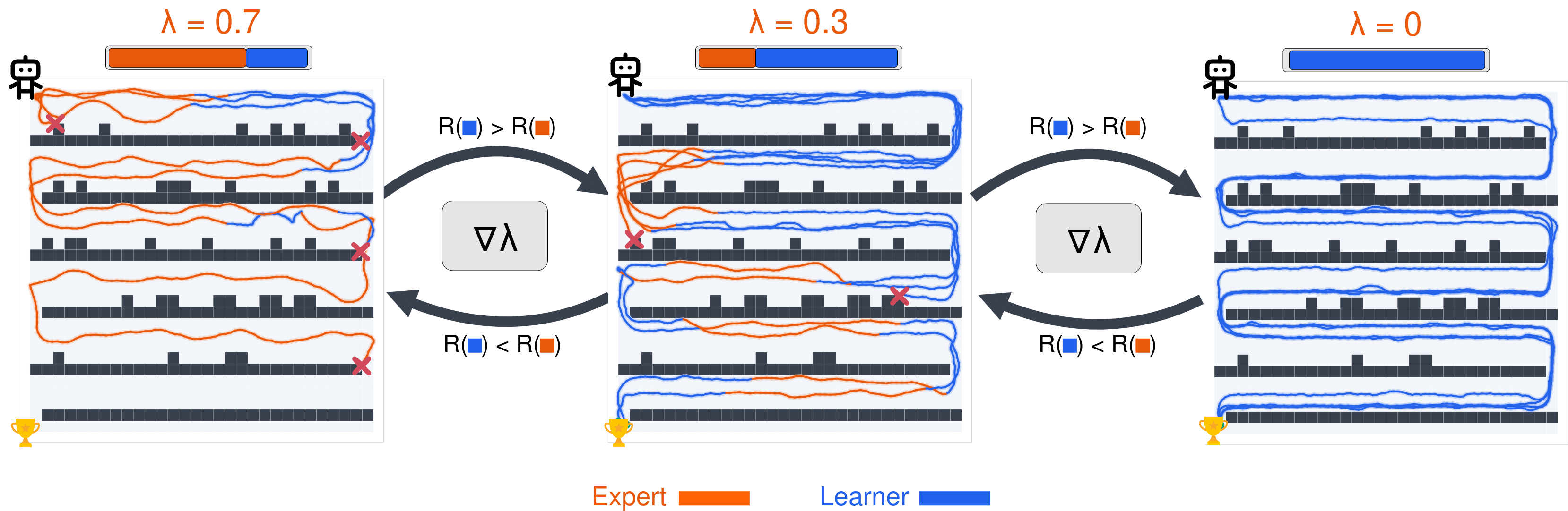}
    \caption{\textbf{The expert share is learned from block-level performance.}
The expert (orange) and learner (blue) alternate control in temporally committed blocks, with $\lambda$ determining the expert's share.
The difference between expert- and learner-block advantages drives the update of $\lambda$: as the learner improves, the expert is progressively withdrawn until $\lambda=0$.}
    \label{fig:teaser}
    \vspace{-1.3em}
\end{figure}

Our approach belongs to the policy gradient family of algorithms, with the learner's parameters optimized from on-policy rollouts. The key difference is that the expert and the learner take turns controlling the agent during these rollouts. As shown in Fig.~\ref{fig:teaser}, we split each training episode into blocks of a fixed number of steps; at the start of each block, a random draw decides whether the expert or the learner acts for the entire block. We refer to the probability of selecting the expert for a given block as the \emph{expert share}, which is automatically adjusted as training proceeds. It starts high, so that the expert's blocks can bring the agent to states the learner is unlikely to reach on its own. For instance, in tasks with sparse rewards, these are states where rewards appear and generate the non-zero gradients the learner can benefit from. As the learner becomes more proficient, the expert share progressively diminishes, eventually vanishing altogether.

This block structure is what makes the automated adjustment of the expert share informative. Because each block is executed entirely by either the learner or the expert, we can attribute its outcome to one of the two, and compare how well the expert's blocks and the learner's blocks perform within the same episode. We update the expert share with the same policy-gradient rule that trains the policy network weights. These updates have an intuitive interpretation: they raise the share when the expert's blocks outperform the learner's and lower it once the learner catches up. The expert thus hands over control as the learner improves, without requiring a predetermined schedule. Blocks are necessary because switching at every step would blur this comparison, since the outcome of a single action depends mostly on the actions that follow it. Two additional mechanisms complete the method. First, we impose a cost on every block played by the expert, so that the expert retains control only while it is clearly better than the learner. This is necessary because, without it, the expert would keep its share as long as the learner merely matches it, and the learner might never act alone. Second, on the expert's blocks, we train the learner to imitate the expert's actions, so that the experience the expert collects also contributes to training the learner; this ensures fuller use of each episode for learner training. Once the expert share reaches zero, the expert is no longer used. Training continues as in ordinary policy gradient, and the final learner acts entirely on its own.
Sect.~\ref{sec:mixture} formalizes the problem. Sect.~\ref{sec:gradients} derives and interprets the gradients for the learner and the expert share. Sect.~\ref{sec:price-imitation} adds the cost and the imitation term.

\subsection{Problem Formulation}
\label{sec:mixture}

Consider a \ac{MDP} $(\cS, \cA, P, r, \rho_0, \gamma)$ with state space $\cS \subseteq \R^n$, action space $\cA \subseteq \R^m$, transition kernel $P(\cdot \mid s, a)$, reward $r \colon \cS \times \cA \to \R$, initial state distribution $\rho_0$, and discount $\gamma \in [0,1)$. Let $\piE$ be a fixed \emph{expert} policy and $\pinn$ a \emph{learner} policy with parameters $\theta \in \R^d$. The objective of learning is to maximize the expected discounted return under the learner, i.e., $\Ex_{\tau\sim \pinn}\big[{\textstyle\sum_{t\ge0}}\gamma^t r_t\big]$.

We partition each episode into blocks of $k$ steps and denote $b_j=jk$ for the start time of block $j$. At the beginning of each block, a switch variable $z_j\sim\operatorname{Bernoulli}(\lambda)$, where $\lambda$ is the expert share, determines whether the expert or the learner rolls out that block. The switches are independent across blocks and independent of the states. Formally, for $z_j\in\{0, 1\}$, the rollout within block $j$ is given by
\begin{equation}
  a_{b_j+i}\sim\pi_{z_j}(\cdot\mid s_{b_j+i}), \quad s_{b_j+i+1} \sim P(\cdot \mid s_{b_j+i}, a_{b_j+i}),
  \quad i\in\{0,\ldots,k-1\},
  \label{eq:commit}
\end{equation}
where $\pi_1\coloneqq \piE$ and $\pi_0 \coloneqq \pinn$. We constrain $\lambda\in[0,\lambda_{\max}]$ with $\lambda_{\max}<1$, so that the learner has a nonzero probability of acting in every block. We refer to the policy yielding these interleaved rollouts as the \emph{mixture policy}. Fig.~\ref{fig:blocks} illustrates the block structure. Unless indicated otherwise, expectations are over the joint distribution of the switches $\bm z=(z_0,z_1,\ldots)$ and the trajectories $\tau=(s_0,a_0,s_1,a_1,\ldots)$ they induce, with $s_0\sim\rho_0$.
\definecolor{expertC}{HTML}{EA580C}
\definecolor{learnerC}{HTML}{2261EB}

\begin{figure}[t]
\centering
\begin{tikzpicture}[
  font=\footnotesize,
  st/.style   ={circle,draw=black!55,fill=white,inner sep=0pt,minimum size=3.5pt},
  bst/.style  ={circle,draw=black,fill=black,inner sep=0pt,minimum size=4.5pt},
  arr/.style  ={-{Stealth[length=3pt]},black!55,line width=0.5pt},
  Eband/.style={fill=expertC!12,draw=expertC,rounded corners=2pt,line width=0.6pt},
  Aband/.style={fill=learnerC!10,draw=learnerC,rounded corners=2pt,line width=0.6pt},
  Gband/.style={fill=black!6,draw=black!45,dashed,rounded corners=2pt,line width=0.6pt},
  brcup/.style={decorate,decoration={brace,amplitude=3pt,raise=1pt},
                black!55,line width=0.5pt},
  alab/.style ={black!55,font=\scriptsize,inner sep=1pt,above=0.5pt},
  slab/.style ={font=\scriptsize,inner sep=1pt,below=3pt},
  zlab/.style ={font=\scriptsize,inner sep=1.5pt,anchor=south east},
]

\draw[Eband] (-0.42,-0.45) rectangle (2.42,0.42);
\draw[Aband] ( 2.58,-0.45) rectangle (5.42,0.42);
\draw[Eband] ( 5.58,-0.45) rectangle (8.42,0.42);
\draw[Gband] ( 9.98,-0.45) rectangle (12.82,0.42);

\foreach \t in {1,2,4,5,7,8} \node[st]  (s\t) at (\t,0) {};
\foreach \t in {0,3,6}       \node[bst] (s\t) at (\t,0) {};
\foreach \t [evaluate=\t as \u using int(\t+1)] in {0,...,7}
  \draw[arr] (s\t) -- (s\u);
\draw[arr] (s8) -- (8.85,0);

\node[bst] (t0) at (10.4,0) {};
\node[st]  (t1) at (11.4,0) {};
\node[st]  (t2) at (12.4,0) {};
\node[bst] (t3) at (13.4,0) {};
\draw[arr] (9.55,0) -- (t0);
\foreach \t [evaluate=\t as \u using int(\t+1)] in {0,1,2}
  \draw[arr] (t\t) -- (t\u);
\node[black!55] at (9.2,0) {$\cdots$};

\node[alab] at (0.5,0)  {$a_{0}$};
\node[alab] at (2.5,0)  {$a_{k-1}$};
\node[alab] at (10.9,0) {$a_{b_j}$};
\node[alab] at (12.9,0) {$a_{b_j+k-1}$};

\draw[brcup] (-0.42,0.44) -- (2.42,0.44)
  node[midway,above=5pt,font=\scriptsize,inner sep=0pt] {$k$ steps};

\node[zlab,expertC]  at (2.42,-0.45)  {$z_0=1$};
\node[zlab,learnerC] at (5.42,-0.45)  {$z_1=0$};
\node[zlab,expertC]  at (8.42,-0.45)  {$z_2=1$};
\node[zlab,black!55] at (12.82,-0.45) {$z_j$};

\node[slab] at (s0) {$s_{0}$};
\node[slab] at (s3) {$s_{k}$};
\node[slab] at (s6) {$s_{2k}$};
\node[slab] at (t0) {$s_{b_j}$};
\node[slab] at (t3) {$s_{b_{j+1}}$};

\end{tikzpicture}
\caption{\textbf{Temporally committed mixture.} The horizon is split into blocks of $k$ steps, starting at $b_j=jk$ (filled states). At each boundary, a switch $z_j\sim\operatorname{Bernoulli}(\lambda)$ hands the whole block to the expert $\piE$ (orange, $z_j=1$) or to the learner $\pinn$ (blue, $z_j=0$).}
\label{fig:blocks}
\vspace{-1.3em}
\end{figure}
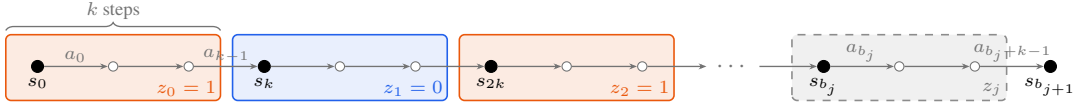
Denote the discounted reward of block $j$ by $\gR_j:=\sum_{i=0}^{k-1}\gamma^i r_{b_j+i}$ and the block discount factor by $\Gamma:=\gamma^k$, and note that the discounted return of an episode can be written with block rewards, $\sum_{t\ge0}\gamma^t r_t = \sum_{j\ge0}\Gamma^{\,j}\gR_j$. We optimize $\theta$ and $\lambda$ for the discounted return of the mixture policy,
\begin{equation}
  J(\lambda,\theta)=\E\Big[{\textstyle\sum_{j\ge0}}\Gamma^{\,j}\gR_j\Big],
  \label{eq:objective}
\end{equation}
which equals the learner's objective when $\lambda=0$. At block boundaries, the rollouts form an \ac{MDP} with state $s_{b_j}$, action $z_j$, reward $\gR_j$, and discount $\Gamma$.

\subsection{Policy gradient estimator under block rollouts}
\label{sec:gradients}

On the \ac{MDP} at block boundaries, define the block-wise switch-value function and value function
\begin{equation}
  Q(s,z):=\E\Big[\sum_{m\ge0}\Gamma^{\,m}\gR_{j+m}\;\Big|\;s_{b_j}=s,\ z_j=z\Big], \qquad V(s)=\lambda Q(s,1)+(1-\lambda)Q(s,0),
  \label{eq:value}
\end{equation}
where $V(s)$ is derived by taking the expectation of $Q(s,z)$ with respect to $z_j$. The block advantage $A(s,z):=Q(s,z)-V(s)$ compares committing the block at $s$ to policy $z$ with the average over the switch. We estimate it with the block temporal-difference residual $\widehat A_j:=\gR_j+\Gamma\,V(s_{b_{j+1}})-V(s_{b_j})$.

The gradients of the objective~\eqref{eq:objective} with respect to the learner parameters $\theta$ and the expert share $\lambda$ are given by

\noindent
\begin{minipage}{0.63\textwidth}
\small
\begin{equation}
\boxed{\gradstrut
\nabla_\theta J
=\E\big[\sum_{j\ge0}\Gamma^{\,j}(1-z_j)\widehat A_j
\sum_{i=0}^{k-1}\nabla_\theta\log\pinn(a_{b_j+i}\mid s_{b_j+i})\big]
}
\label{eq:theta-gradient}
\end{equation}
\end{minipage}%
\begin{minipage}{0.37\textwidth}
\small
\begin{equation}
\boxed{\gradstrut
\frac{\partial J}{\partial\lambda}
=\E\big[\sum_{j\ge0}\Gamma^{\,j}
\frac{z_j-\lambda}{\lambda(1-\lambda)}\widehat A_j\big]
}
\label{eq:lambda-gradient}
\end{equation}
\end{minipage}

Eq.~\ref{eq:theta-gradient} implies that the learner is updated only on the blocks it rolled out, with the block advantage in place of the per-step advantage, and Prop.~\ref{prop:contrast} below gives the $\lambda$-gradient a simple reading. When $\lambda=0$, we have $z_j=0$ for every block, and \eqref{eq:theta-gradient} reduces to the standard actor-critic gradient for $\pinn$ at block resolution. At this point, the block structure is unnecessary: no switching occurs and a single policy acts throughout, so training can revert to per-step updates with the per-step advantage. We now derive eq.~\ref{eq:theta-gradient} and \ref{eq:lambda-gradient}.

\paragraph{Derivation.} Write $\tau_j:=(a_{b_j},s_{b_j+1},\ldots,a_{b_{j+1}-1},s_{b_{j+1}})$ for the trajectory within block $j$, excluding the initial boundary state. The joint density of $\tau_j$ and the switch $z_j$, conditioned on the boundary state $s_{b_j}$ is given by
\begin{equation}
  p(z_j,\tau_j\mid s_{b_j})
  =\underbrace{\lambda^{z_j}(1-\lambda)^{1-z_j}}_{p(z_j)}
  \cdot\underbrace{\prod_{i=0}^{k-1}
    \pi_{z_j}(a_{b_j+i}\mid s_{b_j+i})\,
    P(s_{b_j+i+1}\mid s_{b_j+i},a_{b_j+i})}_{p(\tau_j \mid s_{b_j},z_j)},
  \label{eq:block-law}
\end{equation}
where the first factor does not depend on $s_{b_j}$ because the switch is independent of the state. The joint density of the trajectory $\tau$ and the switches $\bm z$ is then the block-wise product $p(\bm z,\tau)=\rho_0(s_0)\prod_{j\ge0}p(z_j,\tau_j\mid s_{b_j})$. Writing $\log\pi_{z_j}(a \mid s)=z_j\log\piE(a \mid s)+(1-z_j)\log\pinn(a \mid s)$ and noting that $\nabla_\theta\log\piE(a \mid s)=0$ for the expert for all $a, s$, it follows that the score for $\theta$ is
\begin{equation}
  \nabla_\theta \log p(\bm{z}, \tau) = \sum_{j\geq 0} \nabla_\theta\log p(z_j,\tau_j\mid s_{b_j}) = \sum_{j\geq 0} (1-z_j) \sum_{i=0}^{k-1}\nabla_\theta\log\pinn(a_{b_j+i}\mid s_{b_j+i}),
  \label{eq:learner-score}
\end{equation}
and the score for $\lambda$ is
\begin{equation}
    \frac{\partial\log p(\bm z,\tau)}{\partial\lambda}
    =\sum_{j\ge0}\frac{z_j-\lambda}{\lambda(1-\lambda)} .
    \label{eq:lambda-score}
\end{equation}
Neither score contains the expert's action probabilities, so the expert can be a black box that only returns actions. Writing $\nabla$ for either $\nabla_\theta$ or $\partial_\lambda$, the policy gradient theorem~\citep{sutton2000policy} yields $\nabla J=\E\big[\sum_{j\ge0}\Gamma^{\,j}\,\nabla\log p(z_j,\tau_j\mid s_{b_j})\,\widehat A_j\big]$, and substituting eq.~\ref{eq:learner-score} and \ref{eq:lambda-score} yields~\eqref{eq:theta-gradient} and~\eqref{eq:lambda-gradient}.

\paragraph{Reading the $\lambda$-gradient.} The $\lambda$-gradient in~\eqref{eq:lambda-gradient} has a simple reading once the switch is marginalized out. Let
\begin{equation}
  a_E(s):=\E\big[\widehat A_j\mid s_{b_j}=s,\ z_j=1\big],
  \qquad
  a_A(s):=\E\big[\widehat A_j\mid s_{b_j}=s,\ z_j=0\big]
  \label{eq:conditional-advantages}
\end{equation}
be the expected advantage of a block that starts in $s$ and is executed by the expert or by the learner.

\begin{proposition}
\label{prop:contrast}
Because the switch is independent of the boundary state,
\begin{equation}
  \frac{\partial J}{\partial\lambda}
  =\sum_{j\ge0}\Gamma^{\,j}\,
    \E_{s_{b_j}}\big[a_E(s_{b_j})-a_A(s_{b_j})\big],
  \label{eq:lambda-contrast}
\end{equation}
where the expectation is over the marginal distribution of the boundary states that the mixture visits.
\end{proposition}

We prove Prop.~\ref{prop:contrast} in App.~\ref{app:contrast}. Because the switch ignores the state, the expert's blocks and the learner's blocks start from the same distribution of states, so~\eqref{eq:lambda-contrast} compares the two policies on equal footing: the gradient is positive when, on the states the mixture visits, a block executed by the expert earns more in expectation than one executed by the learner, and negative otherwise. Gradient ascent therefore raises $\lambda$ while the expert is ahead on average and lowers it once the learner is ahead. Since $a_E(s)=A(s,1)$ and $a_A(s)=A(s,0)$, the difference also equals $Q(s,1)-Q(s,0)$: the extra return from $s$ when the expert, not the learner, plays the next block.

\subsection{Penalty on expert use and imitation}
\label{sec:price-imitation}

\paragraph{Penalizing expert use.}
Without a penalty, $\lambda$ decreases only when the learner's blocks strictly outperform those of the expert. Once the learner matches the expert, which the imitation term below encourages, the gradient~\eqref{eq:lambda-gradient} vanishes and $\lambda$ remains at its current value, potentially preventing the learner from acting independently. We therefore maximize
\begin{equation}
  \Phi(\lambda,\theta)=J(\lambda,\theta)-c_E\,\sigma\,\lambda,\qquad c_E>0,
  \label{eq:priced-objective}
\end{equation}
where $\sigma$ is the sample standard deviation of the block advantage over a batch, treated as a constant in the gradient, and $c_E$ is the price of expert use in units of $\sigma$. Expressing the price in units of $\sigma$ avoids tuning it per task. The price does not depend on $\theta$, so it changes only the $\lambda$-gradient: the expert share increases only while the expert's blocks outperform the learner's on average by more than a margin proportional to $c_E\sigma$, and decreases once the margin falls below it.

\paragraph{Imitation on expert blocks.}
The factor $(1-z_j)$ in~\eqref{eq:theta-gradient} removes the expert's blocks from the learner update, so a fraction $\lambda$ of the collected experience does not train the learner. We use it through an imitation term on those blocks,
\begin{equation}
  \mathcal I(\theta)=\E\Big[\sum_{j\ge0}\Gamma^{\,j}\,z_j\sum_{i=0}^{k-1}\log\pinn(a_{b_j+i}\mid s_{b_j+i})\Big],
  \label{eq:imitation}
\end{equation}
weighted by $c_I\ge0$ and differentiated with the sampling distribution held fixed, as is standard for auxiliary losses. As in DAgger \cite{ross2011dagger}, imitation uses states the learner helps reach rather than fixed demonstrations. Imitating a mediocre expert can hold the learner back, so we set $c_I$ per benchmark (Sect.~\ref{sec:ablations}), and the term fades with the expert because it acts only on expert blocks. We do not differentiate $\mathcal I$ with respect to $\lambda$: doing so would let $\lambda$ rise to generate more blocks to imitate, whereas it should respond only to the block comparison and its price.

\paragraph{Resulting update.}
The learner ascends $J+c_I\,\mathcal I$ and the expert share ascends $\Phi$ so the complete gradients are:

\begin{equation}
\begin{aligned}
\nabla_\theta\big(J+c_I\,\mathcal I\big)
  &=\E\Big[\sum_{j\ge0}\Gamma^{\,j}\big[(1-z_j)\,\widehat A_j+c_I\,z_j\big]\sum_{i=0}^{k-1}\nabla_\theta\log\pinn(a_{b_j+i}\mid s_{b_j+i})\Big],\\
\frac{\partial\Phi}{\partial\lambda}
  &=\E\Big[\sum_{j\ge0}\Gamma^{\,j}\,\frac{z_j-\lambda}{\lambda(1-\lambda)}\,\widehat A_j\Big]-c_E\,\sigma,
\end{aligned}
\label{eq:priced-gradient}
\end{equation}
with $\lambda$ projected onto $[0,\lambda_{\max}]$ after each update. On learner blocks, the learner follows the policy gradient; on expert blocks, it imitates the expert. Once $\lambda=0$, no expert block is sampled, the imitation term vanishes, and training is standard policy gradient.


\section{Experiments}
\label{sec:experiments}

We address four questions.
\textbf{(Q1)} Does the learned handover yield higher returns with fewer environment steps than guided and unguided baselines? (Sect.~\ref{sec:results})
\textbf{(Q2)} Does $\lambda$ correlate well with the expert's declining superiority over the learner and eventually vanish, as \eqref{eq:priced-gradient} predicts? (Sect.~\ref{sec:mpc})
\textbf{(Q3)} Can a weak expert that often fails help and enable the learner to surpass it? (Sect.~\ref{sec:frozenlake})
\textbf{(Q4)} How much does each hyper-parameter, the block length $k$, the imitation weight $c_I$, and the expert price $c_E$, matter? (Sect.~\ref{sec:ablations})

\subsection{Experimental Setup}

\paragraph{Benchmarks.}
We evaluate on 34 tasks from two common benchmarks, both in the GPU-parallel implementations that motivate our on-policy setting. DMControl~\citep{tunyasuvunakool2020}, through MuJoCo Playground~\citep{zakka2025mujoco}, covers continuous control at varying difficulty, including sparse-reward tasks. MiniGrid~\citep{MinigridMiniworld23}, through NAVIX~\citep{pignatelli2026navix}, adds sparse-reward gridworlds with discrete actions. Each method trains for 20M environment steps on 5 seeds per task.
\vspace{-0.8em}

\paragraph{Baselines.}
All methods share the same \ac{PPO} learner, network, and optimizer settings (App.~\ref{sec:hyperparameters}). They differ only in how they use the expert. \emph{\ac{PPO}}~\citep{schulman2017proximal} is the unguided baseline. \emph{\acs{BC} $\rightarrow$ \acs{PPO}} follows \citet{ramrakhya2023pirlnav} in its simplest form: $M$ steps of \ac{BC}, then \ac{PPO}. \emph{Scheduled~$\lambda$} uses the same block structure and imitation component, but we modify the learned share with the schedule $\lambda(t)=\max(0,1-t/T)$, for $T\in\{1, 2.5, 5\}\times10^{6}$ steps as proposed by~\citet{yang2023reinforcement}. \emph{\ac{JSRL}}~\citep{uchendu2023jumpstart} lets the expert act for a prefix of each episode and shortens the prefix as the learner improves. \ac{JSRL} was designed for off-policy learning, so we exclude expert actions from the policy-gradient loss and keep them for value learning and \ac{BC}. \emph{Kickstarting}~\citep{schmitt2018kickstarting} adds a distillation loss toward the expert on the learner's rollouts and with a population-based adaptation.
\vspace{-0.8em}

\paragraph{Experts.}
Our target setting is an expert that comes from outside the training run, such as a controller, a planner, or a heuristic, available as a black box that returns actions. On MiniGrid, the expert is of this kind: a depth-first-search (DFS) planner with injected noise, tuned to reach about 70\% of the maximum return on average across tasks. On DMControl, handcrafting such an expert for each task is impractical, so we emulate it with learned experts: \ac{PPO} checkpoints stopped before convergence and TD-MPC2 policies~\citep{hansen2024tdmpc2}, which reach about 60\% of the best final return of any method on average. Our method never reads their parameters or resumes training from them; it only queries them for actions. App.~\ref{app:guides} reports per-task expert returns, and we will release all experts. Sects.~\ref{sec:mpc} and~\ref{sec:frozenlake} add two model-based experts that plan with inaccurate models: an \ac{MPC} controller with a perturbed dynamics model and a grid planner with a PD tracker.
\vspace{-0.8em}

\begin{figure}[t]
    \centering
    \includegraphics[width=\linewidth]{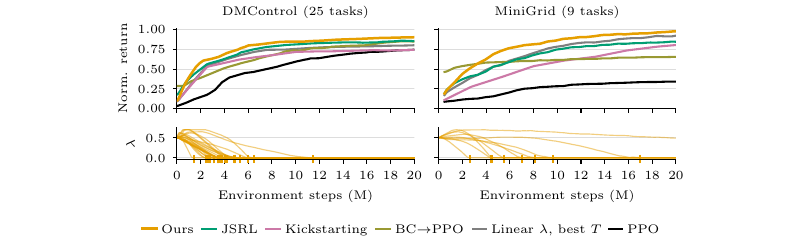}
    \vspace{-6mm}
    \caption{\textbf{Aggregate results on both benchmarks.} Normalized return (min-max per task, averaged over tasks and seeds) and the expert share $\lambda$ per task.}
    \label{fig:aggregate}
\end{figure}

\paragraph{Metrics.}
Asymptotic return is the deterministic return averaged over the final 20\% of training. Steps to bar is the number of environment steps needed to first reach 75\% of the best asymptotic return any method achieved on that task, so the bar is the same for all methods; runs that never reach it count as 20M in medians. In tables, normalized return is the return divided by the best method's return on that task. Aggregate curves normalize each task by its observed range.

\subsection{Comparison of Return Performance}
\label{sec:results}

Tab.~\ref{tab:grouped} summarizes both metrics by task family, and Fig.~\ref{fig:aggregate} shows the aggregate learning curves and the learned expert share $\lambda$; per-task tables, curves, and $\lambda$ trajectories are in App.~\ref{app:results}. On DMControl, our method has the lowest median steps to bar, $4.8$M against $5.6$M for \ac{JSRL}, $6.1$M for \ac{BC}$\rightarrow$\ac{PPO}, $9.1$M for Kickstarting, and $10.1$M for \ac{PPO}, a median per-task speed-up of $2.55\times$ over \ac{PPO}. It also has the highest mean normalized return, and it is the fastest method on 16 of 25 tasks, while no other method is fastest on more than 5 (Tab.~\ref{tab:pertask}). On MiniGrid, reward is sparse, and the methods separate more sharply. \ac{PPO} reaches the bar on only 3 of 9 tasks and ends with zero return on three of the four KeyCorridor variants. \ac{BC}$\rightarrow$\ac{PPO} reaches the bar fastest when cloning covers most of the state space, but it collapses on the two largest KeyCorridor variants and on DoorKey-Random-16x16. Because our method queries the expert on the states the learner itself visits, it avoids this collapse: it obtains the highest final return on 4 tasks and ties for the highest on 3 more (Tab.~\ref{tab:minigrid_pertask}). The scheduled-$\lambda$ arms show why the share has to be learned rather than fixed in advance. No single linear schedule horizon $T$ works across all tasks at the same time. Our method beats every baseline on both metrics in every family except on Finger and steps to bar in MiniGrid, and the $\lambda$ trajectories it learns share no common shape (Fig.~\ref{fig:aggregate}).

\begin{figure}[ht]
    \centering
    \captionsetup{font=small}
    \begin{minipage}[t]{0.46\linewidth}
        \vspace{0pt}\centering
        \captionof{table}{\textbf{Results by task family.} Median steps to bar (M)\,/\,mean normalized return; runs that never reach the bar count as 20M. \textbf{Bold}: best, \underline{underline}: second best. Group definitions and per-task results in App.~\ref{app:results}.}
\label{tab:grouped}
\scriptsize
\setlength{\tabcolsep}{1.5pt}
\begin{tabular}{lccccc}
\toprule
 & \multicolumn{3}{c}{DMControl} & & \\
\cmidrule(lr){2-4}
 & Low-dim. & Finger & Loco. & MiniGrid & All \\
\midrule
\textbf{Ours} & \textbf{1.7}/\textbf{.952} & 7.7/\underline{.981} & \textbf{5.5}/\textbf{.957} & \underline{7.8}/\textbf{.960} & \textbf{5.1}/\textbf{.958} \\
JSRL & \underline{3.3}/.902 & \underline{7.4}/.838 & \underline{5.6}/\underline{.933} & 7.9/.835 & \underline{5.8}/\underline{.889} \\
Kick & 4.2/.820 & 14.8/.744 & 16.0/.781 & 16.7/.775 & 10.1/.789 \\
BC$\rightarrow$PPO & 5.7/.920 & \textbf{5.7}/\textbf{.989} & 9.2/.856 & \textbf{4.3}/.649 & 6.0/.834 \\
$\lambda$1M & \underline{3.3}/\underline{.944} & 8.7/.891 & 7.4/.726 & 8.7/.771 & 7.0/.823 \\
$\lambda$2.5M & 4.6/\underline{.944} & \underline{7.4}/.942 & 7.4/.683 & 9.1/.906 & 6.6/.849 \\
$\lambda$5M & 5.0/\underline{.944} & 10.3/.934 & 10.3/.718 & 12.4/\underline{.907} & 8.7/.860 \\
PPO & 6.0/.807 & 10.1/.899 & 12.6/.730 & 20.0/.336 & 11.9/.665 \\
\bottomrule
\end{tabular}

    \end{minipage}\hfill
    \begin{minipage}[t]{0.52\linewidth}
        \vspace{0pt}\centering
        \includegraphics[width=\linewidth]{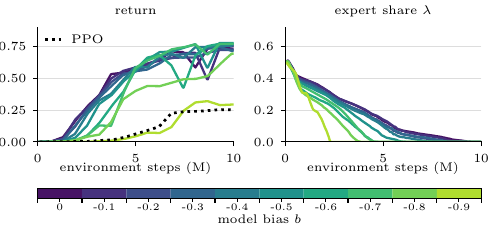}
        \caption{CartpoleSwingupSparse with an \ac{MPC} expert whose model masses are scaled by $(1+b)$. \textbf{Left:} learner return. \textbf{Right:} learned expert share $\lambda$. $b=0$ is the exact model; as the error grows, $\lambda$ reaches zero sooner.}
        \label{fig:ddp_model_bias}
    \end{minipage}
\end{figure}
\vspace{-4mm} 
\subsection{Does \texorpdfstring{$\lambda$}{lambda} track expert quality?}
\label{sec:mpc}

Sect.~\ref{sec:results} varies the expert and the task together, so here we fix the task and vary only the expert's quality, using a different type of expert to also test generality. Following \citet{bhardwaj2021blending}, we use a \ac{MPPI} controller on CartpoleSwingupSparse whose internal model scales the cart and pole masses by $(1+b)$, with $b\in\{0,-0.1,\ldots,-0.9\}$; at $b=0$, the model is exact. The environment itself never changes. Fig.~\ref{fig:ddp_model_bias} shows that the handover adapts to expert quality: $\lambda$ reaches zero at $9.8$M steps for the exact model and at $2.2$M for $b=-0.9$, with the intermediate levels ordered in between. 
No hyperparameters change across the sweep. The learner reaches about 0.7 normalized return at every level of mismatch except $b=-0.9$. There, the expert helps little, the handover withdraws it within $2.2$M steps, and the learner ends still above \ac{PPO}: with a poor expert, the method automatically falls back to \ac{PPO} rather than being held back.

\begin{figure}[h]
    \centering
    \includegraphics[width=0.9\linewidth]{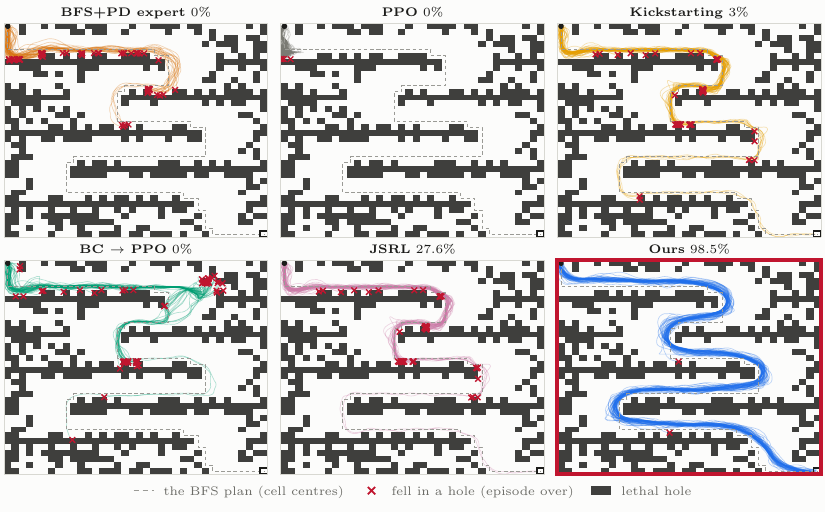}
    \vspace{-4mm}
        \caption{\textbf{Frozen Lake with a failing expert} ($p_{\text{slip}}=0.3$). 60 evaluation episodes per method on the same map, with the success rate above each panel. Only our learner reliably completes the route.}
    \label{fig:frozen-lake}
\end{figure}

\subsection{Can a failing expert help?}
\label{sec:frozenlake}

In the above examples, we make the experts suboptimal on purpose. Here we explore a scenario common in practice: the expert can fail because it plans with an oversimplified model. In our continuous version of the \emph{Frozen Lake Problem}, a point mass must cross a $36\times36$ lake by commanding 2D accelerations. A third of the cells are holes that end the episode, some lined up in rows with alternating gaps that force a serpentine route. Reward comes only at the goal. The ice is slippery: with probability $p_{\text{slip}}=0.3$, the commanded acceleration is rotated by $\pm90^\circ$. The expert uses a breadth-first search PD controller to track the shortest grid path. Because this path grazes the holes and the plan ignores the slip, the expert never reaches the goal (Fig.~\ref{fig:frozen-lake}, full setup in App.~\ref{app:frozenlake}). Exploration is hard in this continuous task, with \ac{PPO} never leaving the start. Our learner keeps the expert's route but learns to avoid the holes. It reaches the goal in 98.5\% of the episodes, against 4\% for Kickstarting, 28\% for \ac{JSRL}, and 0\% for \ac{BC}$\rightarrow$\ac{PPO}. Thus, a failing expert can help if the learner may deviate where it fails.

\subsection{Ablations}
\label{sec:ablations}
\begin{figure}[h]
    \centering
    \begin{subfigure}[t]{0.32\linewidth}
        \centering
        \includegraphics[width=\linewidth]{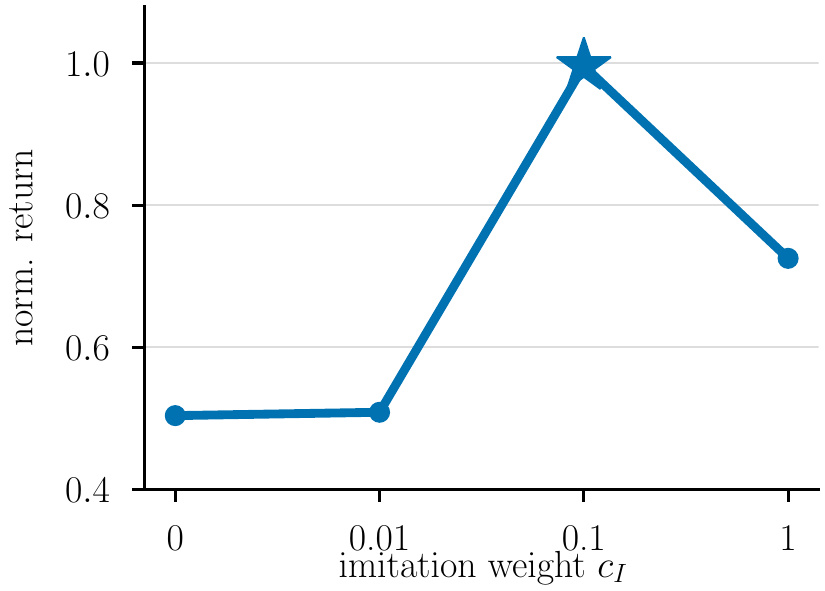}
        \caption{}
        \label{fig:imitation-ablation}
    \end{subfigure}%
    \hfill
    \begin{subfigure}[t]{0.32\linewidth}
        \centering
        \includegraphics[width=\linewidth]{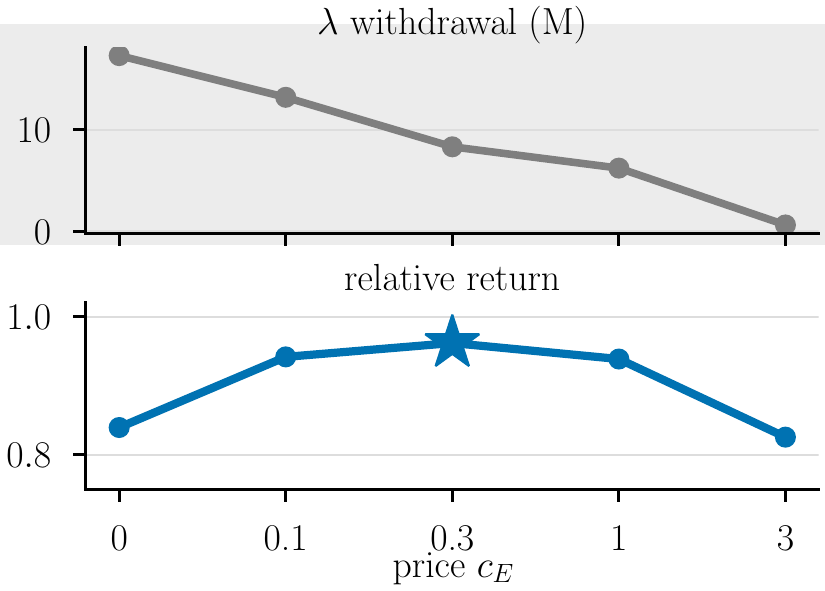}
        \caption{}
        \label{fig:price-ablation}
    \end{subfigure}%
    \hfill
    \begin{subfigure}[t]{0.32\linewidth}
        \centering
        \includegraphics[width=\linewidth]{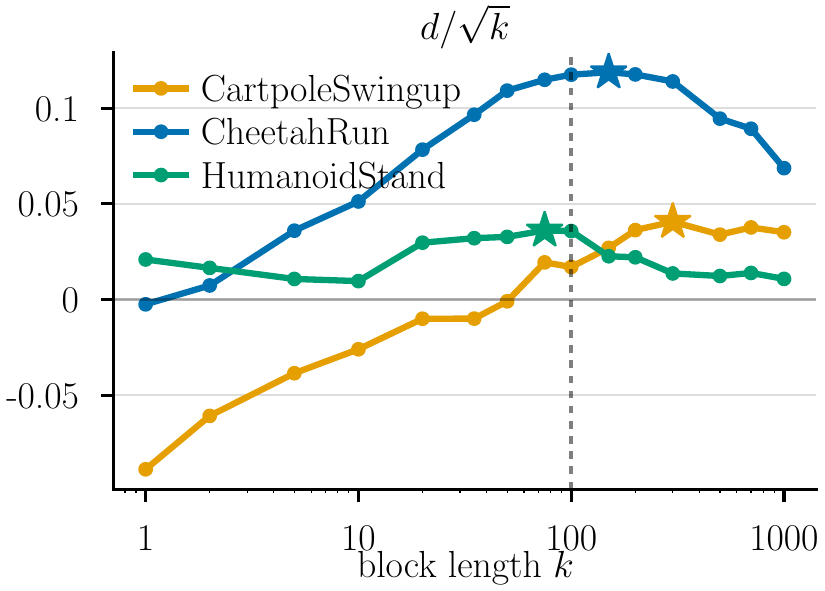}
        \caption{}
        \label{fig:snr-k1000}
    \end{subfigure}
    \vspace{-2mm}
    \caption{\textbf{Hyperparameter ablations.} (a) Imitation weight: normalized return vs.\ $c_I$ (Fig.~\ref{fig:imit-pertask}). (b) Expert price: withdrawal step of $\lambda$ (top) and relative return (bottom) vs.\ $c_E$ (Fig.~\ref{fig:price-pertask}). (c) Block size trade-off: dotted line marks $k{=}100$. Stars mark the maximum. Per-task plots in App.~\ref{app:ablations-pertask}.}
\end{figure}
\vspace{-0.8em}
\paragraph{Imitation regularization.} We sweep $c_I \in \{0, 0.01, 0.1, 1\}$ (Fig.~\ref{fig:imitation-ablation}). The term matters where exploration is hard: for instance, on the Humanoid tasks, removing it leaves the learner at the floor for the entire budget. Excessive imitation instead anchors the learner to the expert and hinders exploration, so $c_I=0.1$ gives the best trade-off across tasks (per-task curves in Fig.~\ref{fig:imit-pertask}).
\paragraph{Block size.}
By Prop.~\ref{prop:contrast}, $\lambda$ moves with the gap $a_E-a_A$. We measure its visibility on a block by the per-block signal-to-noise ratio $d(k):=\E[a_E-a_A]/\sigma$, where $\sigma$ is the standard deviation of the block advantage (untrained learner, ${\lambda=0.5}$, no critic). Maximizing $d$ alone favors long blocks, but a batch of $N$ steps contains $N/k$ blocks, giving gradient signal-to-noise $d(k)\sqrt{N/k}$. Thus, at fixed budget, we maximize $d(k)/\sqrt{k}$ (Fig.~\ref{fig:snr-k1000}). At $k=1$, $d$ is near zero or negative, so single steps provide little information about which policy is better; it then grows with $k$ on every task. In contrast, $d(k)/\sqrt{k}$ peaks at intermediate lengths, $k\approx75$--$300$, with our $k=100$ near the peak.
\paragraph{Price on expert use.} Without the price ($c_E=0$), $\lambda$ stays above 0 on 13 of 25 DMC tasks, even once the learner matches the expert, capping performance at the expert's level. Raising $c_E$ withdraws the expert sooner: mean withdrawal falls from $17.4$M steps at $c_E=0$ to $0.6$M at $c_E=3$. However, too high a price withdraws it before it pays off, so relative return peaks at the default $c_E=0.3$, the value we use on every task without per-task tuning (Fig.~\ref{fig:price-ablation} and~\ref{fig:price-pertask}).

\section{Conclusion}
\label{sec:conclusion}
We presented a policy-gradient method in which a fixed expert and the learner share each episode in blocks, with the expert's share optimized by the same objective as the policy. Committing to blocks reduces the $\lambda$-gradient to the gap between what expert and learner blocks earn. A price on expert use turns withdrawal into part of the objective, and an imitation term on expert blocks lets every block train the learner. Our method reaches a given return in fewer environment steps than guided and unguided baselines, with a median per-task speed-up of $2.55\times$ over \ac{PPO} on DMControl and the largest margins under sparse rewards, even with an expert that never succeeds. The handover needs no per-task tuning, and once $\lambda$ reaches zero the deployed policy runs without the expert. Two directions follow. The comparison of block advantages extends to several experts, each with its own share. And since $\lambda$ measures what the expert is worth, it could budget queries to expensive experts, such as the model-based controllers common in robotics.

\subsection*{AI Use Statement}
We used generative AI tools to polish the writing of the manuscript and as a coding assistant. The authors reviewed all AI-edited text and tested all AI-assisted code, and take full responsibility for the content of this work.

\subsection*{Reproducibility Statement}
Sec.~\ref{sec:method} and App.~\ref{app:contrast} give the derivations and proofs. Sec.~\ref{sec:experiments} and App.~\ref{sec:hyperparameters} describe the setup, hyperparameters, and experts, and App.~\ref{app:frozenlake} describes the FrozenLake task. We will release the code and all experts upon publication.

\bibliography{lib.bib}
\bibliographystyle{preprint}

\appendix

\section{Proof of Proposition~\ref{prop:contrast}}
\label{app:contrast}

\begin{proof}
The switch $z_j$ is drawn independently of the boundary state, so $(s_{b_j},z_j)$ has a product law, and the expectation in~\eqref{eq:lambda-gradient} can be taken in three stages: over the boundary state $s_{b_j}$ under the mixture, then over $z_j\sim\operatorname{Bernoulli}(\lambda)$, which the conditioning leaves unchanged, then over the execution of the block:
\begin{equation}
  \E\Big[\frac{z_j-\lambda}{\lambda(1-\lambda)}\,\widehat A_j\Big]
  =\E_{s_{b_j}}\Big[\,\E_{z_j}\Big[\,
      \E\Big[\frac{z_j-\lambda}{\lambda(1-\lambda)}\,\widehat A_j
        \;\Big|\;s_{b_j},z_j\Big]\Big]\Big].
  \label{eq:nested}
\end{equation}
In the innermost expectation, the score depends only on $z_j$, so it moves outside:
\begin{equation}
  \E\Big[\frac{z_j-\lambda}{\lambda(1-\lambda)}\widehat A_j
    \;\Big|\;s_{b_j}=s,\,z_j\Big]
  =\frac{z_j-\lambda}{\lambda(1-\lambda)}\;
   \E\big[\widehat A_j\mid s_{b_j}=s,z_j\big]
  =\begin{cases}
     \dfrac{a_E(s)}{\lambda}, & z_j=1,\\[8pt]
     \dfrac{-\,a_A(s)}{1-\lambda}, & z_j=0 .
   \end{cases}
  \label{eq:inner}
\end{equation}
Averaging over $z_j$, each outcome's probability cancels the denominator of its own branch:
\begin{equation}
  \lambda\cdot\frac{a_E(s)}{\lambda}
   +(1-\lambda)\cdot\frac{-a_A(s)}{1-\lambda}
  =a_E(s)-a_A(s).
  \label{eq:conditional-contrast}
\end{equation}
Substituting into~\eqref{eq:lambda-gradient} for every block gives~\eqref{eq:lambda-contrast}.
\end{proof}

\section{Additional Results}
\label{app:results}

\paragraph{Task groups.}
Table~\ref{tab:grouped} summarizes the per-task results of Tables~\ref{tab:pertask} and~\ref{tab:minigrid_pertask}. It splits the DMControl tasks into three families:
\emph{Low-dim.\ control} (11 tasks): AcrobotSwingup, AcrobotSwingupSparse, CartpoleBalance, CartpoleBalanceSparse, CartpoleSwingup, CartpoleSwingupSparse, PendulumSwingup, BallInCup, PointMass, ReacherEasy, ReacherHard;
\emph{Finger} (3 tasks): FingerSpin, FingerTurnEasy, FingerTurnHard;
\emph{Locomotion} (11 tasks): CheetahRun, FishSwim, HopperHop, HopperStand, SwimmerSwimmer6, WalkerRun, WalkerStand, WalkerWalk, HumanoidRun, HumanoidStand, HumanoidWalk.
The MiniGrid column covers the 9 tasks of Table~\ref{tab:minigrid_pertask}, and the All column covers all 34 tasks. For each column, we report the median steps to bar over tasks, counting a run that does not reach the bar within 20M steps as 20M. We also report the mean over tasks of the return divided by the best method's return on that task, using the asymptotic return on DMControl and the final return on MiniGrid.

\subsection{DMControl}
\label{app:dmc}

Table~\ref{tab:pertask} reports the steps to bar and the asymptotic return on every DMControl task. Fig.~\ref{fig:per-task-dmc} shows the learning curves, and Fig.~\ref{fig:per-task-lambda-dmc} shows the expert share $\lambda$ of our method over training.

\begin{table*}[ht]
\centering
\caption{\textbf{Sample efficiency and final performance per task on DMC.}
Each cell shows \emph{steps to bar} (M) / \emph{asymptotic return} (median over 5 seeds; ``--'' if not reached within 20M steps).
In the median-steps row, runs that do not reach the bar count as 20M. Normalized return is the return divided by the best method's return on that task, averaged over tasks.
\textbf{Bold}: best. \underline{Underline}: second best. Ties are marked for every tied method and counted for each of them in the last row.}
\label{tab:pertask}
\scriptsize
\setlength{\tabcolsep}{3pt}
\begin{tabular}{lrcccccccc}
\toprule
Task & bar & \textbf{Ours} & JSRL & Kick & BC$\rightarrow$PPO & $\lambda$1M & $\lambda$2.5M & $\lambda$5M & PPO \\
\midrule
AcrobotSwingup & 145 & \textbf{5.8}\,/\,175 & 9.3\,/\,156 & 9.1\,/\,152 & \underline{7.7}\,/\,\underline{192} & 8.7\,/\,185 & 7.8\,/\,\textbf{193} & 8.7\,/\,189 & 14.8\,/\,189 \\
AcrobotSwingupSparse & 41 & \textbf{8.6}\,/\,\textbf{54} & 12.3\,/\,36 & --\,/\,27 & 12.4\,/\,50 & 14.0\,/\,\underline{52} & 14.0\,/\,50 & \underline{11.9}\,/\,49 & --\,/\,19 \\
BallInCup & 730 & \textbf{1.9}\,/\,\textbf{974} & 3.3\,/\,\underline{973} & 4.2\,/\,968 & \underline{2.2}\,/\,971 & 2.9\,/\,\underline{973} & 3.3\,/\,\textbf{974} & 4.6\,/\,\underline{973} & 6.6\,/\,\textbf{974} \\
CartpoleBalance & 665 & \underline{1.0}\,/\,854 & 1.3\,/\,822 & 1.3\,/\,834 & \textbf{0.1}\,/\,\underline{884} & 1.3\,/\,817 & 1.3\,/\,814 & 1.7\,/\,821 & 2.5\,/\,\textbf{886} \\
CartpoleBalanceSparse & 750 & \textbf{1.1}\,/\,987 & \underline{3.3}\,/\,\textbf{1000} & 8.7\,/\,743 & 16.5\,/\,642 & 4.2\,/\,996 & 5.0\,/\,\underline{998} & 5.0\,/\,933 & 6.0\,/\,944 \\
CartpoleSwingup & 566 & \textbf{1.5}\,/\,661 & 6.6\,/\,704 & \underline{2.5}\,/\,572 & 5.7\,/\,\underline{719} & 3.3\,/\,697 & 4.6\,/\,690 & \underline{2.5}\,/\,693 & 4.2\,/\,\textbf{754} \\
CartpoleSwingupSparse & 611 & \textbf{4.1}\,/\,783 & \underline{5.8}\,/\,\textbf{815} & 9.1\,/\,432 & 13.7\,/\,635 & 7.4\,/\,717 & 7.4\,/\,729 & 7.0\,/\,\underline{788} & 15.0\,/\,566 \\
CheetahRun & 636 & 5.4\,/\,\underline{829} & \textbf{3.7}\,/\,\textbf{848} & 7.0\,/\,576 & 6.1\,/\,801 & 7.0\,/\,815 & 5.8\,/\,826 & 9.5\,/\,823 & \underline{5.2}\,/\,816 \\
FingerSpin & 727 & 9.3\,/\,914 & \textbf{5.4}\,/\,\textbf{969} & 11.9\,/\,808 & \underline{5.7}\,/\,\underline{962} & 9.1\,/\,792 & 6.2\,/\,926 & 8.7\,/\,931 & 8.0\,/\,923 \\
FingerTurnEasy & 672 & \underline{7.7}\,/\,\textbf{896} & 18.5\,/\,570 & --\,/\,600 & \textbf{5.7}\,/\,874 & 8.7\,/\,777 & 11.9\,/\,787 & 16.0\,/\,773 & 10.1\,/\,\underline{880} \\
FingerTurnHard & 625 & \textbf{5.7}\,/\,\textbf{833} & 7.4\,/\,730 & 14.8\,/\,606 & 7.1\,/\,\underline{831} & \underline{6.2}\,/\,824 & 7.4\,/\,827 & 10.3\,/\,815 & 11.1\,/\,635 \\
FishSwim & 429 & \textbf{4.6}\,/\,\underline{567} & \underline{7.4}\,/\,\textbf{572} & 16.0\,/\,444 & 7.9\,/\,552 & \underline{7.4}\,/\,558 & \underline{7.4}\,/\,550 & 10.3\,/\,544 & 12.6\,/\,471 \\
HopperHop & 283 & \underline{7.2}\,/\,\textbf{378} & \textbf{5.6}\,/\,\underline{374} & 18.5\,/\,276 & --\,/\,137 & 14.8\,/\,262 & 17.8\,/\,288 & 15.2\,/\,304 & 19.7\,/\,134 \\
HopperStand & 647 & \textbf{2.0}\,/\,\underline{778} & 5.0\,/\,768 & 3.3\,/\,\textbf{863} & 9.2\,/\,763 & 5.8\,/\,745 & \underline{2.9}\,/\,766 & 5.8\,/\,770 & --\,/\,436 \\
HumanoidRun & 210 & \underline{7.7}\,/\,\underline{238} & \textbf{5.4}\,/\,\textbf{280} & --\,/\,198 & 14.1\,/\,221 & --\,/\,76 & --\,/\,7 & --\,/\,62 & --\,/\,144 \\
HumanoidStand & 745 & \textbf{5.5}\,/\,\textbf{993} & \underline{5.7}\,/\,754 & 11.1\,/\,892 & 10.6\,/\,\underline{951} & --\,/\,36 & --\,/\,35 & --\,/\,35 & 11.9\,/\,936 \\
HumanoidWalk & 712 & \underline{14.4}\,/\,\underline{768} & \textbf{9.1}\,/\,\textbf{950} & --\,/\,606 & --\,/\,663 & --\,/\,213 & --\,/\,18 & --\,/\,238 & --\,/\,162 \\
PendulumSwingup & 629 & \textbf{4.8}\,/\,828 & 6.2\,/\,698 & \underline{5.0}\,/\,774 & 5.9\,/\,\underline{832} & 7.0\,/\,\textbf{839} & 5.8\,/\,829 & 7.4\,/\,\textbf{839} & --\,/\,63 \\
PointMass & 655 & \textbf{0.1}\,/\,748 & 2.5\,/\,743 & \underline{0.9}\,/\,\textbf{873} & 3.0\,/\,740 & 2.5\,/\,733 & 2.9\,/\,739 & 5.8\,/\,741 & 3.3\,/\,\underline{751} \\
ReacherEasy & 716 & 1.7\,/\,899 & 2.5\,/\,881 & 2.5\,/\,\underline{910} & \textbf{0.6}\,/\,\textbf{955} & 2.5\,/\,864 & 2.5\,/\,872 & 2.5\,/\,865 & \underline{1.3}\,/\,\textbf{955} \\
ReacherHard & 704 & \textbf{1.3}\,/\,926 & 3.3\,/\,922 & 2.5\,/\,836 & \underline{1.6}\,/\,\underline{932} & 2.5\,/\,930 & 4.2\,/\,925 & 4.6\,/\,\textbf{938} & 3.3\,/\,909 \\
SwimmerSwimmer6 & 379 & \textbf{6.9}\,/\,\textbf{505} & \underline{8.7}\,/\,326 & --\,/\,292 & 11.6\,/\,434 & 10.7\,/\,\underline{491} & 11.9\,/\,447 & 16.0\,/\,415 & 16.0\,/\,403 \\
WalkerRun & 524 & 6.0\,/\,\textbf{698} & 5.8\,/\,691 & --\,/\,480 & \textbf{4.0}\,/\,674 & \underline{5.0}\,/\,\underline{696} & 7.0\,/\,679 & 9.1\,/\,667 & 7.2\,/\,672 \\
WalkerStand & 732 & \textbf{1.7}\,/\,\textbf{977} & 2.5\,/\,973 & 2.5\,/\,923 & \underline{2.2}\,/\,973 & 2.5\,/\,973 & 5.4\,/\,\underline{974} & 7.0\,/\,973 & 4.6\,/\,973 \\
WalkerWalk & 721 & \textbf{1.5}\,/\,\underline{961} & \underline{2.5}\,/\,959 & 4.2\,/\,916 & 3.0\,/\,957 & 2.9\,/\,959 & 6.2\,/\,957 & 7.8\,/\,957 & 4.4\,/\,\textbf{962} \\
\midrule
\multicolumn{2}{l}{Median steps / norm.\ return} & \textbf{4.8}\,/\,\textbf{0.957} & \underline{5.6}\,/\,\underline{0.908} & 9.1\,/\,0.794 & 6.1\,/\,0.901 & 7.0\,/\,0.841 & 6.2\,/\,0.829 & 8.7\,/\,0.843 & 10.1\,/\,0.784 \\
\multicolumn{2}{l}{\# fastest / \# best return} & \textbf{16}\,/\,\textbf{9} & \underline{5}\,/\,\underline{7} & 0\,/\,2 & 4\,/\,1 & 0\,/\,1 & 0\,/\,2 & 0\,/\,2 & 0\,/\,5 \\
\bottomrule
\end{tabular}
\end{table*}

\begin{figure}[p]
    \centering
    \includegraphics[width=\linewidth]{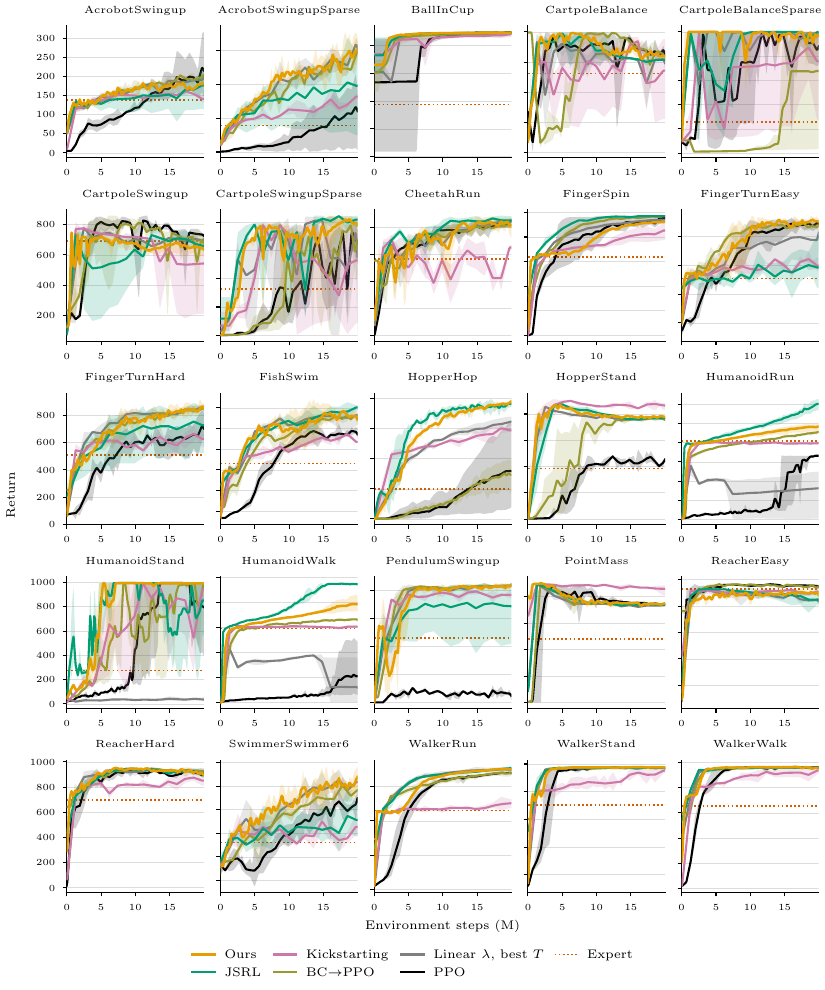}
    \caption{\textbf{DMControl learning curves per task.} Return of the learner acting alone with deterministic actions, for each method.}
    \label{fig:per-task-dmc}
\end{figure}

\begin{figure}[p]
    \centering
    \includegraphics[width=\linewidth]{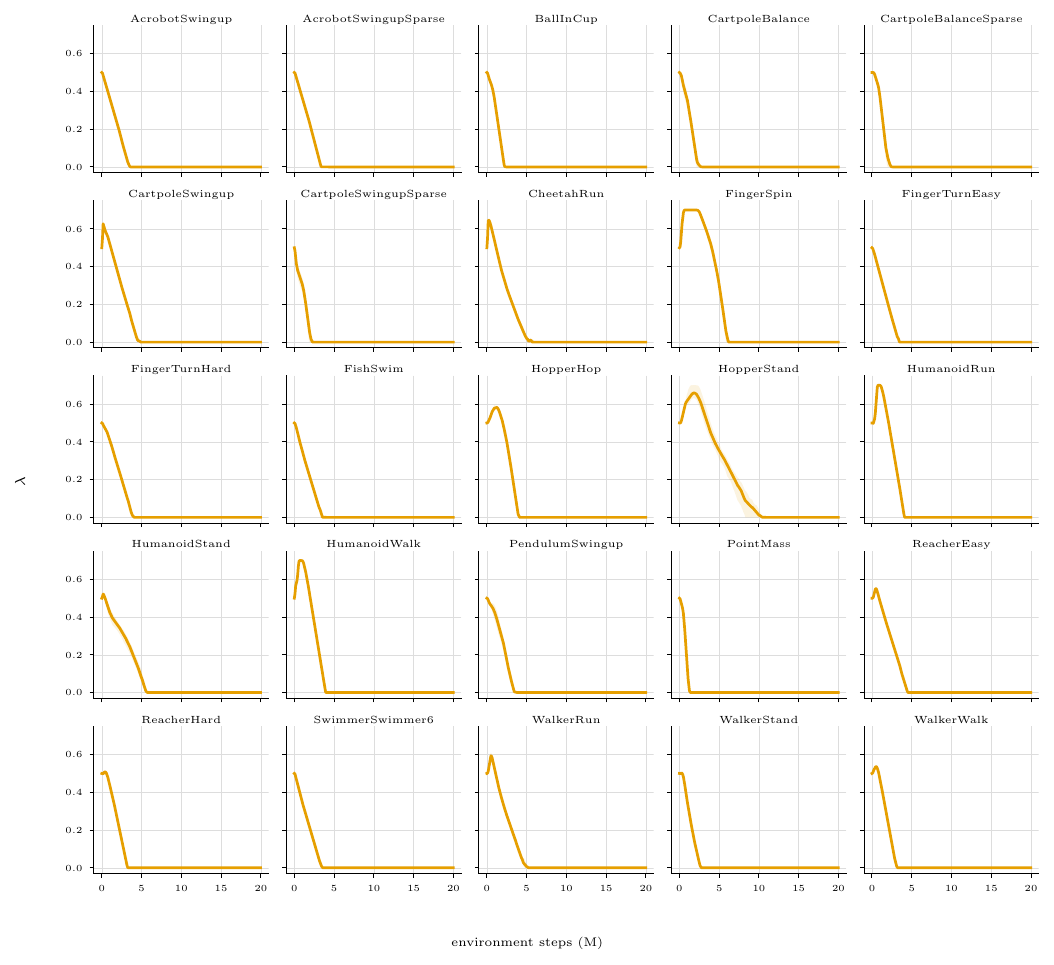}
    \caption{\textbf{DMControl expert share per task.} Expert share $\lambda$ of our method over training.}
    \label{fig:per-task-lambda-dmc}
\end{figure}

\subsection{MiniGrid}
\label{app:minigrid}

Table~\ref{tab:minigrid_pertask} reports the steps to bar and the final return on every MiniGrid task. Fig.~\ref{fig:per-task-mg} shows the learning curves, with the expert's return as a dotted line, and Fig.~\ref{fig:per-task-lambda-mg} shows the expert share $\lambda$.

\begin{table*}[ht]
\centering
\caption{\textbf{Sample efficiency and final performance per task on MiniGrid.}
Each cell shows \emph{steps to bar} / \emph{final return}.
Steps to bar is the number of environment steps (in millions) needed to reach 75\% of the best final return on that task (median over 5 seeds; ``--'' if not reached within 20M steps).
Final return is the deterministic return after 20M steps (mean over 5 seeds).
In the median-steps row, runs that do not reach the bar count as 20M. Normalized return is the final return divided by the best method's final return on that task, averaged over tasks.
\textbf{Bold}: best. \underline{Underline}: second best. Ties are marked for every tied method and counted for each of them in the last row.}
\label{tab:minigrid_pertask}
\scriptsize
\setlength{\tabcolsep}{3pt}
\begin{tabular}{lrcccccccc}
\toprule
Task & bar & \textbf{Ours} & JSRL & Kick & BC$\rightarrow$PPO & $\lambda$1M & $\lambda$2.5M & $\lambda$5M & PPO \\
\midrule
KeyCorridorS4R3 & 0.75 & 4.7\,/\,0.98 & 5.9\,/\,\underline{0.99} & 16.7\,/\,0.75 & \textbf{2.2}\,/\,0.94 & \underline{4.2}\,/\,0.92 & 4.9\,/\,\textbf{1.00} & 7.4\,/\,\underline{0.99} & --\,/\,0.00 \\
KeyCorridorS5R3 & 0.72 & \underline{8.5}\,/\,\textbf{0.96} & \textbf{7.9}\,/\,0.92 & 17.9\,/\,0.69 & --\,/\,0.00 & 9.1\,/\,0.77 & 9.1\,/\,\underline{0.94} & 12.4\,/\,0.90 & --\,/\,0.00 \\
KeyCorridorS6R3 & 0.67 & \textbf{7.8}\,/\,\textbf{0.90} & \underline{12.3}\,/\,\underline{0.77} & 15.7\,/\,0.75 & --\,/\,0.00 & --\,/\,0.00 & --\,/\,0.46 & 14.8\,/\,\underline{0.77} & --\,/\,0.00 \\
DoorKey-16x16 & 0.70 & \textbf{8.5}\,/\,\textbf{0.93} & 10.5\,/\,0.79 & 16.7\,/\,0.69 & \underline{8.6}\,/\,0.81 & 9.0\,/\,\underline{0.90} & 9.6\,/\,\underline{0.90} & 12.5\,/\,0.77 & --\,/\,0.17 \\
DoorKey-Random-16x16 & 0.52 & \textbf{10.3}\,/\,\textbf{0.69} & --\,/\,0.15 & --\,/\,0.34 & --\,/\,0.09 & --\,/\,0.21 & 15.8\,/\,0.49 & \underline{14.5}\,/\,\underline{0.61} & --\,/\,0.03 \\
LavaCrossingS11N5 & 0.48 & --\,/\,0.42 & 12.5\,/\,0.42 & --\,/\,0.30 & \textbf{4.3}\,/\,\underline{0.58} & \underline{8.7}\,/\,0.56 & 10.9\,/\,\textbf{0.64} & 18.9\,/\,0.45 & --\,/\,0.06 \\
FourRooms-13x13 & 0.75 & \underline{0.7}\,/\,\textbf{1.00} & \underline{0.7}\,/\,\textbf{1.00} & 1.9\,/\,\textbf{1.00} & \textbf{0.5}\,/\,\textbf{1.00} & \underline{0.7}\,/\,\textbf{1.00} & \underline{0.7}\,/\,\textbf{1.00} & \underline{0.7}\,/\,\textbf{1.00} & 8.2\,/\,\textbf{1.00} \\
KeyCorridorS3R3 & 0.75 & 2.1\,/\,\textbf{1.00} & \underline{1.9}\,/\,\textbf{1.00} & 6.7\,/\,\underline{0.98} & \textbf{0.5}\,/\,\textbf{1.00} & 3.0\,/\,\textbf{1.00} & 3.8\,/\,\textbf{1.00} & 4.6\,/\,\textbf{1.00} & 6.9\,/\,0.77 \\
FourRooms-17x17 & 0.75 & \underline{1.0}\,/\,\textbf{1.00} & 1.2\,/\,\underline{0.99} & 2.6\,/\,\textbf{1.00} & \textbf{0.5}\,/\,\textbf{1.00} & 1.5\,/\,\textbf{1.00} & 1.6\,/\,\textbf{1.00} & \underline{1.0}\,/\,\textbf{1.00} & 11.9\,/\,0.93 \\
\midrule
\multicolumn{2}{l}{Median steps / norm.\ return} & \underline{7.8}\,/\,\textbf{0.960} & 7.9\,/\,0.835 & 16.7\,/\,0.775 & \textbf{4.3}\,/\,0.649 & 8.7\,/\,0.771 & 9.1\,/\,0.906 & 12.4\,/\,\underline{0.907} & 20.0\,/\,0.336 \\
\multicolumn{2}{l}{\# fastest / \# best return} & \underline{3}\,/\,\textbf{7} & 1\,/\,2 & 0\,/\,2 & \textbf{5}\,/\,3 & 0\,/\,3 & 0\,/\,\underline{5} & 0\,/\,3 & 0\,/\,1 \\
\bottomrule
\end{tabular}
\end{table*}

\begin{figure}[p]
    \centering
    \includegraphics[width=\linewidth]{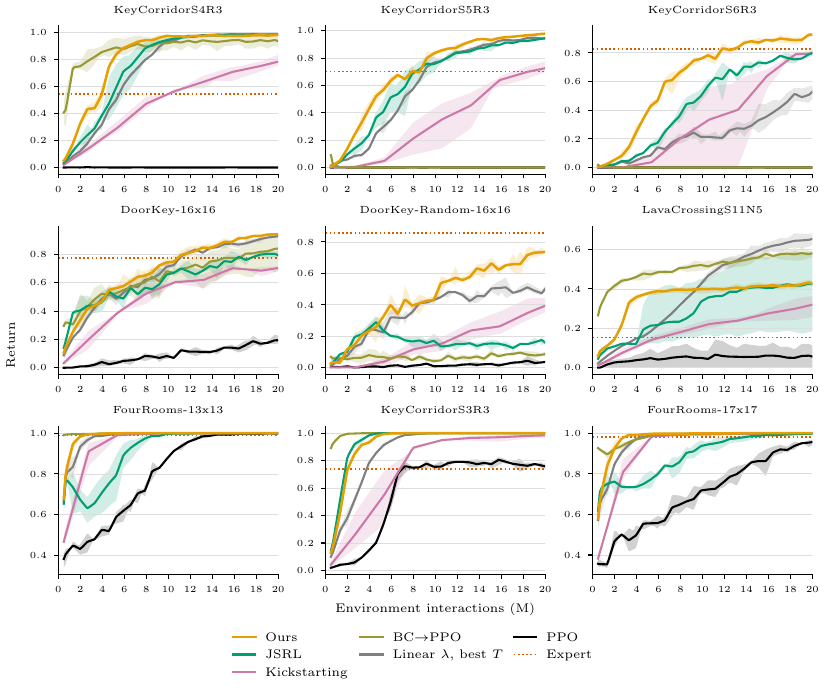}
    \caption{\textbf{MiniGrid learning curves per task.} Return of the learner acting alone with deterministic actions, for each method. The dotted line is the expert's return.}
    \label{fig:per-task-mg}
\end{figure}

\begin{figure}[p]
    \centering
    \includegraphics[width=\linewidth]{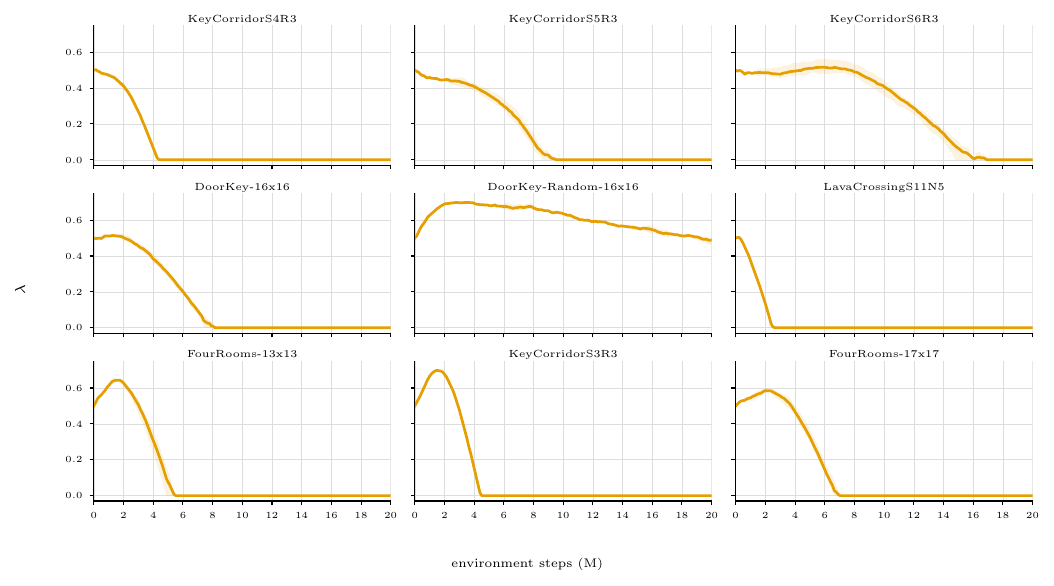}
    \caption{\textbf{MiniGrid expert share per task.} Expert share $\lambda$ of our method over training.}
    \label{fig:per-task-lambda-mg}
\end{figure}

\clearpage
\subsection{FrozenLake}
\label{app:frozenlake}

\paragraph{Task.}
A point mass moves on a continuous $36\times36$ grid in which $32.9\%$ of the cells are holes. The action is a two-dimensional acceleration. Entering a hole ends the episode. The reward is $+1$ on the step that enters the goal and $0$ otherwise. With probability $p_{\text{slip}}$, the commanded acceleration is rotated by a quarter turn. Five full-width walls with alternating doorways force a serpentine route of $123$ cells with $25$ turns. Episodes last at most $700$ steps.

\paragraph{Expert.}
Breadth-first search over the free cells returns a path, and a PD controller tracks it. The expert is suboptimal by construction: the shortest path passes close to the corners of holes, and the planner has no model of the slip. It never reaches the goal (Fig.~\ref{fig:frozen-lake}).

\paragraph{Protocol.}
All methods share the environment, the expert, the network, the \ac{PPO} hyperparameters, and the budget of $2\times10^{7}$ transitions. They differ only in how they use the expert. We report the success rate of the learner acting alone with deterministic actions.

\paragraph{Result.}
Table~\ref{tab:flake_seeds} and Fig.~\ref{fig:frozen-lake} give the results at $p_{\text{slip}}=0.3$. \ac{PPO} never reaches the goal, so its advantages are zero for the whole run and the reward gives it no learning signal. \ac{BC}$\rightarrow$\ac{PPO} follows a route that ends in a hole. \ac{JSRL} and Kickstarting get further along the route than the expert, but they usually fall into a hole before the goal. Our learner reaches the goal in $98.5\%$ of the episodes.

\begin{table}[h]
\centering
\caption{\textbf{Final success rate on FrozenLake} ($p_{\text{slip}}=0.3$), mean over 5 seeds.}
\label{tab:flake_seeds}
\begin{tabular}{lc}
\toprule
Method & Success rate \\
\midrule
\ac{PPO}                   & 0.000 \\
\ac{BC}$\rightarrow$\ac{PPO} & 0.000 \\
Kickstarting               & 0.037 \\
\ac{JSRL}                  & 0.276 \\
Ours                       & \textbf{0.985} \\
\bottomrule
\end{tabular}
\end{table}

\clearpage
\subsection{Ablations per task}
\label{app:ablations-pertask}

Fig.~\ref{fig:imit-pertask} and Fig.~\ref{fig:price-pertask} show per-task runs for the imitation-weight and expert-price ablations of Sec.~\ref{sec:ablations} (Figs.~\ref{fig:imitation-ablation} and~\ref{fig:price-ablation}).

\begin{figure}[h]
    \centering
    \includegraphics[width=\linewidth]{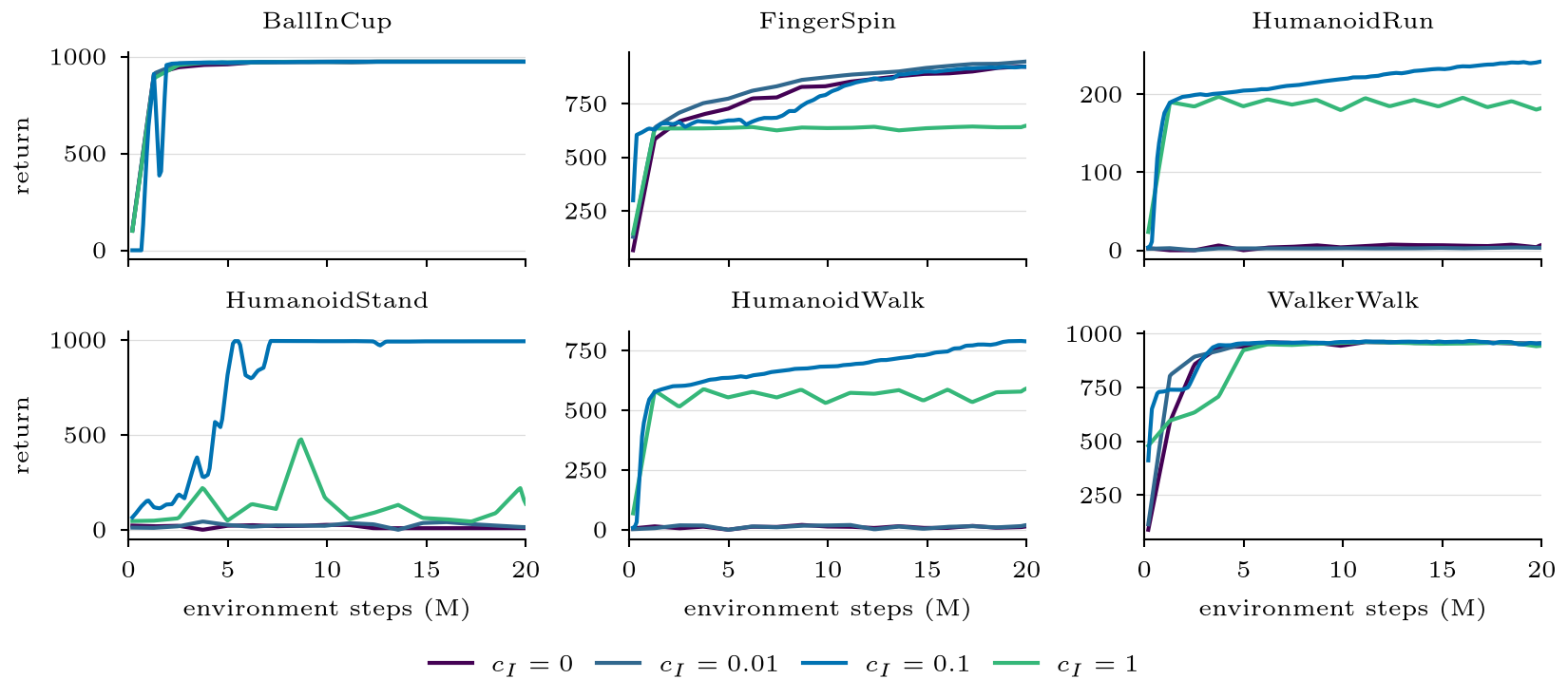}
    \caption{\textbf{Imitation weight per task.} Return over training on six DMControl tasks for each $c_I$. With $c_I\in\{0,0.01\}$, the learner stays near zero return on the three Humanoid tasks. With $c_I=1$, it stays near the expert's return on FingerSpin and ends lower on the Humanoid tasks. On BallInCup and WalkerWalk, $c_I$ has little effect.}
    \label{fig:imit-pertask}
\end{figure}

\begin{figure}[h]
    \centering
    \includegraphics[width=\linewidth]{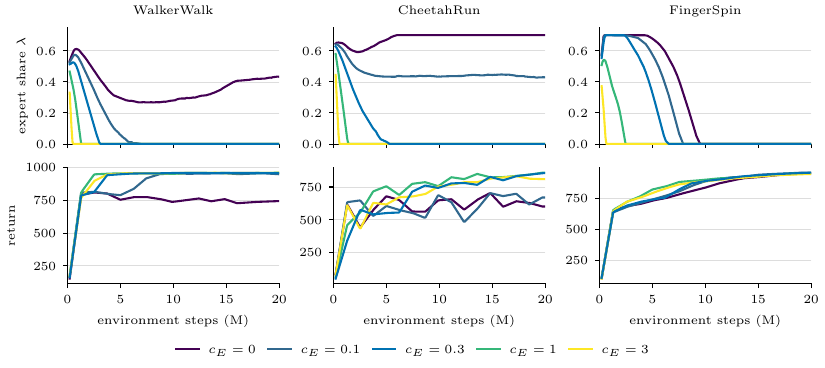}
    \caption{\textbf{Expert price per task.} Expert share $\lambda$ (top) and return (bottom) over training on three DMControl tasks for each $c_E$. With $c_E=0$, $\lambda$ stays above zero until the end of training on WalkerWalk and CheetahRun, and reaches zero on FingerSpin after about 10M steps. With $c_E=3$, $\lambda$ reaches zero within 1M steps on all three tasks.}
    \label{fig:price-pertask}
\end{figure}

\section{Hyperparameters}
\label{sec:hyperparameters}

Table \ref{tab:hyperparams} reports the hyperparameters used for DMControl, while Table \ref{tab:hyperparams_minigrid} reports only the differences with respect to DMControl. All baselines use the same network architecture, with two hidden layers of 256 units for both the actor and critic, and $\tanh$ activations.

\begin{table}[ht]
\centering
\scriptsize
\setlength{\tabcolsep}{3pt}
\caption{\textbf{Hyperparameters.} The PPO backbone is identical for all methods; only the guidance mechanism differs. All methods use the same frozen guide $\pi_E$ with no per-task tuning.}
\label{tab:hyperparams}
\begin{tabular}{@{}ll@{}}
\toprule
\textbf{Parameter} & \textbf{Value} \\
\midrule
\multicolumn{2}{@{}l}{\textit{Shared PPO backbone}} \\
$N$ parallel environments & 2048 \\
Rollout length $T$ & 30 \\
Minibatches / epochs & 32 / 16 \\
$\gamma$ / $\lambda_{\mathrm{GAE}}$ & 0.995 / 0.95 \\
Clip $\epsilon$ & 0.2 \\
Value / entropy coefficient & 0.5 / 0.01 \\
Max gradient norm & 0.5 \\
Activation & tanh \\
Observation normalisation & yes \\
Environment steps & $2\times10^7$ \\
Evaluation & 128 $\times$ 1000, deterministic \\
\midrule
\multicolumn{2}{@{}l}{\textit{Ours}} \\
Block length $k$ & 100 \\
$\lambda_0$ / $\lambda_{\max}$ & 0.5 / 0.7 \\
$\eta_\lambda$ / $c_E$ / $c_I$ & 0.03 / 0.3 / 0.1 \\
\midrule
\multicolumn{2}{@{}l}{\textit{JSRL}} \\
Tolerance $\tau$ / stages & 0.2 / 10 \\
Guide horizon / MA window & 1000 / 0.3 \\
BC coefficient & 0.1 \\
\midrule
\multicolumn{2}{@{}l}{\textit{BC$\rightarrow$PPO}} \\
Transitions / epochs & $5\times10^5$ / 60 \\
Learning rate / batch size & $3\times10^{-4}$ / 4096 \\
\midrule
\multicolumn{2}{@{}l}{\textit{Kickstarting}} \\
Distillation coefficient & 1.0 \\
Anneal fraction / teacher samples & 0.5 / 8 \\
PBT population & 4 \\
\midrule
\multicolumn{2}{@{}l}{\textit{Linear $\lambda$}} \\
Schedule & $1\!\rightarrow\!0$ \\
Decay horizon & 1M (also 2.5M, 5M) \\
Imitation weight $c_I$ & 0.1 \\
\bottomrule
\end{tabular}
\end{table}

\begin{table}[ht]
\centering
\scriptsize
\setlength{\tabcolsep}{3pt}
\caption{\textbf{MiniGrid hyperparameters.} Only differences from the DMControl settings in Table~\ref{tab:hyperparams} are reported.}
\label{tab:hyperparams_minigrid}
\begin{tabular}{@{}ll@{}}
\toprule
\textbf{Parameter} & \textbf{Value} \\
\midrule
\multicolumn{2}{@{}l}{\textit{Shared PPO backbone}} \\
$N$ parallel environments & 512 \\
Rollout length $T$ & 64 \\
Minibatches / epochs & 8 / 2 \\
Learning rate & $10^{-3}$ \\
Observation normalisation & no \\
\midrule
\multicolumn{2}{@{}l}{\textit{Ours}} \\
Block length $k$ & 30 \\
Imitation weight $c_I$ & 1.0 \\
\midrule
\multicolumn{2}{@{}l}{\textit{JSRL}} \\
Tolerance $\tau$ & 0.05 \\
BC coefficient & 1.0 \\
\midrule
\multicolumn{2}{@{}l}{\textit{Linear $\lambda$}} \\
Imitation weight $c_I$ & 1.0 \\
\bottomrule
\end{tabular}
\end{table}

\section{Guides}
\label{app:guides}
Tables \ref{tab:experts} and \ref{tab:experts_minigrid} provide the reward achieved by each guide.
\begin{table}[t]
\centering
\small
\caption{Guide quality on DMControl. The ceiling is the best asymptotic return of any method on that task (Table~\ref{tab:pertask}). Guides are deliberately sub-optimal (mean $0.60\times$ the ceiling, range $0.26$--$0.97$).}
\label{tab:experts}
\begin{tabular}{@{}lrrr@{}}
\toprule
Task & Guide $\pi_E$ & ceiling & $\pi_E/\text{ceil}$ \\
\midrule
AcrobotSwingup & 139.1 & 193 & 0.72 \\
AcrobotSwingupSparse & 15.7 & 54 & 0.29 \\
BallInCup & 844.0 & 974 & 0.87 \\
CartpoleBalance & 724.9 & 886 & 0.82 \\
CartpoleBalanceSparse & 256.2 & 1000 & 0.26 \\
CartpoleSwingup & 693.1 & 754 & 0.92 \\
CartpoleSwingupSparse & 328.9 & 815 & 0.40 \\
CheetahRun & 565.0 & 848 & 0.67 \\
FingerSpin & 635.2 & 969 & 0.66 \\
FingerTurnEasy & 509.1 & 896 & 0.57 \\
FingerTurnHard & 510.5 & 833 & 0.61 \\
FishSwim & 330.3 & 572 & 0.58 \\
HopperHop & 100.4 & 378 & 0.27 \\
HopperStand & 386.1 & 863 & 0.45 \\
HumanoidRun & 203.7 & 280 & 0.73 \\
HumanoidStand & 273.4 & 993 & 0.28 \\
HumanoidWalk & 594.8 & 950 & 0.63 \\
PendulumSwingup & 460.0 & 839 & 0.55 \\
PointMass & 482.0 & 873 & 0.55 \\
ReacherEasy & 927.4 & 955 & 0.97 \\
ReacherHard & 694.8 & 938 & 0.74 \\
SwimmerSwimmer6 & 261.1 & 505 & 0.52 \\
WalkerRun & 461.0 & 698 & 0.66 \\
WalkerStand & 705.1 & 977 & 0.72 \\
WalkerWalk & 654.1 & 962 & 0.68 \\
\midrule
\textit{mean} & & & 0.60 \\
\bottomrule
\end{tabular}
\end{table}
\begin{table}[t]
\centering
\small
\caption{Guide quality on MiniGrid. Guides are deliberately sub-optimal (mean $0.73\times$ the attainable performance).}
\label{tab:experts_minigrid}
\begin{tabular}{@{}lrrr@{}}
\toprule
Task & Guide $\pi_E$ & ceiling & $\pi_E/\text{ceil}$ \\
\midrule
KeyCorridorS4R3 & 0.546 & 1.000 & 0.55 \\
KeyCorridorS5R3 & 0.703 & 1.000 & 0.70 \\
KeyCorridorS6R3 & 0.825 & 1.000 & 0.83 \\
DoorKey-16x16 & 0.776 & 1.000 & 0.78 \\
DoorKey-Random-16x16 & 0.857 & 1.000 & 0.86 \\
LavaCrossingS11N5 & 0.154 & 1.000 & 0.15 \\
FourRooms-13x13 & 0.989 & 1.000 & 0.99 \\
KeyCorridorS3R3 & 0.740 & 1.000 & 0.74 \\
FourRooms-17x17 & 0.981 & 1.000 & 0.98 \\
\midrule
\textit{mean} & & & 0.73 \\
\bottomrule
\end{tabular}
\end{table}

\end{document}